\RequirePackage{fix-cm}
\let\oldvec\vec
\documentclass[smallextended,nospthms]{svjour3}
\let\vec\oldvec
\usepackage{hyperref}
\usepackage{caption,subcaption}    
\usepackage[utf8]{inputenc}
\usepackage{amsmath}
\usepackage{amssymb}
\usepackage{enumerate}
\usepackage[english]{babel}
\usepackage{amsthm}
\newtheorem{theorem}{Theorem}
\newtheorem{lemma}{Lemma}
\usepackage{mathtools}
\usepackage{multicol}
\usepackage{graphicx}
\usepackage{verbatim}
\smartqed  

\usepackage{algorithm}
\usepackage{algpseudocode}
  
\DeclareMathOperator{\Tr}{Tr}

\begin{document}
\title{Optimal bandwidth estimation for a fast manifold learning algorithm to detect circular structure in high-dimensional data
}
\titlerunning{Bandwidth Estimation for Detecting Circularity with Fast Manifold Learning}       
\protect{\author{Susovan Pal \and Praneeth Vepakomma}}
\protect{\institute{S. Pal \at
              Department of Neurology\\
            University of California Los Angeles\\
635 Charles E. Young Drive South\\
Los Angeles, CA 90095, USA \\
              \email{susovan97@gmail.com}          
           \and
           P. Vepakomma \at
              Department of Statistics (previous)\\
              Rutgers University \\
              501 Hill Center\\
                110 Frelinghuysen Road\\
                Piscataway, NJ 08854\\
              Motorola Solutions, Inc (current)\\
              500 W. Monroe Street\\ 
              Chicago, IL 60661, USA\\
              \email{praneeth@scarletmail.rutgers.edu}
            }}
\date{Received: date / Accepted: date}
\maketitle
\begin{abstract}
We provide a way to infer about existence of topological circularity in high-dimensional data sets in $\mathbb{R}^d$ from its projection in $\mathbb{R}^2$ obtained through a fast manifold learning map as a function of
the high-dimensional dataset $\mathbb{X}$ and a particular choice of a positive real $\sigma$ known as bandwidth
parameter. At the same time we also provide a way to estimate the optimal bandwith for fast
manifold learning in this setting through minimization of these functions of bandwidth. We also
provide limit theorems to characterize the behavior of our proposed functions of bandwidth.

\end{abstract}

\section{Introduction}
High-dimensional datasets occur naturally in many areas such as computer vision, computational neuroscience, computational biology, speech analysis, localization of wireless sensor networks, graph visualization, high-dimensional data fusion, spatio-temporal data analysis to name a few. In many instances of such datasets the dimensionality is only artificially high, though each data point consists of perhaps thousands of variables, it may be described as a function of only a few underlying parameters. That is, the data points are actually samples from a low-dimensional manifold that is embedded in a high-dimensional space. The problem of manifold learning is concerned with finding low-dimensional representations of high-dimensional data. Manifold learning algorithms attempt to uncover this low-dimensional representation of data. Many seemingly complex systems described by high-dimensional data sets are in fact governed by a surprisingly low number of parameters. Revealing the low-dimensional representation of such high-dimensional
data sets not only leads to a more compact description of the data, but also enhances our understanding of the system.  Laplacian Eigenmaps\cite{LaplacianEig}, Isomap\cite{Isomap}, Diffusion Maps \cite{DiffusionMaps}, Hessian Eigenmaps\cite{HessianEig}, Local Tangent Space Alignment\cite{LTSA}, Locally Linear Embedding\cite{LLE}, Continuum Isomap\cite{CIsomap}, Maximum Variance Unfolding\cite{MVU}, t-distributed Stochastic Neighborhood Embedding\cite{tSNE}, Semidefinite Embedding\cite{SemidefiniteEmbedding} and more recently Fast Manifold Learning with SDD Linear Systems\cite{FastSDD} are some examples of such mapping techniques for manifold learning. These techniques are also referred to as `Nonlinear Dimensionality Reduction' techniques. 
\subsection{Manifold Learning on Circle Homeomorphisms:}
 In this paper we attempt to see how well the topological structure of the dataset is preserved when we apply the 'Fast SDD Manifold Learning Map' \cite{FastSDD}. To be more specific, we look at a special setting of applying the 'Fast SDD manifold Learning Map' \cite{FastSDD} on datasets that lie on a topological circle. We would like to rigorize the reference to the term 'topological   circle' by stating that a topological circle in $\mathbb{R}^d$ is any homeomorphic image of $\mathbb{S}^1$. A topological circle in $\mathbb{R}^d$, is defined as a one-to-one continuous image of the map $j$ from a circle $ \mathbb{S}^1$ to a high-dimensional space $\mathbb{R}^d$ denoted by $j:\mathbb{S}^1 \to \mathbb{R}^d$. Let's assume we were given a uniform sample of $n$ points lying on this continuous image of $j$ which is the dataset matrix. We denote this dataset by a real matrix $\mathbf{X}_{n\times d}$. For any chosen bandwidth $\sigma \in \mathbb{R}^+$ and the dataset matrix $\mathbf{X}$ with data lying in $\mathbb{R}^d$, we denote the Fast SDD Manifold Learning Map as $\mathbf{Z} \equiv \mathbf{Z(\mathbf{X},\sigma)}: \mathbb{R}^{n\times d}\times (0,\infty) \to   \mathbb{R}^{n \times 2}$.  In order to infer whether the result of applying  Fast SDD Manifold Learning Map \cite{FastSDD} on a discrete dataset of points in $\mathbb{R}^d$ is a topological circle we inturn check whether the result of applying \cite{FastSDD} produces a set of points $\mathbf{Z}(\mathbf{X},\sigma)$ that form a polygon or not. This is motivated by that fact that with increasing number of edges a regular polygon becomes closer and closer to a circle. We therefore can infer on whether the dataset in $\mathbb{R}^d$ lies on a topological circle or not based on whether the result obtained by the Fast SDD Manifold Learning Map is a polygon or not. One primary concern here is the fact that Fast SDD Manifold Learning Map as well as other manifold learning maps are parametrized by a scalar bandwidth parameter $\sigma$. We provide a method to optimally choose $\sigma$ by minimization of a $L^2$  energy function $\mathbf{E}(\mathbf{X},\sigma)$ that we propose. We show that for a fixed $\mathbf{X} \in \mathbb{R}^{d}$, the $\mathbf{Z}(\mathbf{X},\sigma)$ corresponding to a $\sigma$ that takes the values close enough to $\underset{\sigma}{\mathrm{ arg \enskip min }}\enspace\mathbf{E}(\sigma)$; lies on a polygon. This way we have a method to infer on the existence of topological circularity on any given dataset in $\mathbb{R}^d$ by computing the minima $\sigma_*$ of $\mathbf{E}(\mathbf{X},\sigma)$. 
 \par

\section{Two motivating use-cases in computer vision and computational biology:}
In practice high-dimensional data in $\mathbb{R}^d$ with an intrinsic circular geometric representation in a lower dimension occurs commonly in areas of computer vision and computational biology. We now refer and point to such specific use-cases in these domains.

\subsection{Topological Circularity for Human Motion Analysis in Computer Vision} With regards to analysis of human motion via data captured through computer vision, the works of \cite{HMA,HMA2,HMA3,HMA4,HMA5} show intuitively that the gait is a 1-dimensional manifold which is embedded in a high dimensional 'visual variable' space. As an example, Figure 3 in \cite{HMA6} shows that computer vision data collected for modeling the task of recognizing human activities lies on 2 dimensional 'loop' like geometries. In addition to this the human motion data naturally has an ordering associated with it as the human subject progresses from the beginning of his gait sequence towards the completion of his gait sequence in a fixed order of gait actions with respect to time. 

\subsection{Topological Circularity for Tracking Resilience to Infections in Computational Biology} 
Very recently, researchers studied gene expression data collected through cross-sectional and longitudinal studies for people at different stages of malaria. Upon examining this data using 'Topological Data Analysis' \cite{Malaria}  they found that all patients (hosts) data lies on a loop or circle sitting inside of a high dimensional space. They were able to characterize the resilience of hosts to Malarial infection by finding that resilient hosts tend to have their mapped data lying on small loops whereas non-resilient individuals end up getting mapped into large loops.

 \section{Presence of an order on data}
 In this setting, we assume that the points or the data are collected in an order. As we see in the motivated use-cases in 
previous section, or the biological experiments described later in this paper in section 9, this is a natural assumption in many cases. In this setting, we assume that the points in dataset $\mathbf{X}$ are collected with respect to an order in $\mathbb{S}^1$.  More
rigorously, if $n$ points ${\mathbf{X_1}, \mathbf{X_2},...\mathbf{X_n}}$ were uniformly sampled from it, then there exists $t_1<t_2< \ldots <t_n$ so that $j(t_r)=\mathbf{X}_r \forall r \in \{1,2 \ldots n\}$ where $\mathbf{X}_r$ denotes the $r$-th point of $\mathbf{X}$ .  Hence we can form a data matrix X whose $r$-th
row is $X_t=j(t_r)$. Note that, interchanging the order of data will do so for the rows of X.

\section{Fast SDD Manifold Learning Map}
In this section we introduce the Fast SDD Manifold Learning Map $\mathbf{Z}(\mathbf{X},\sigma)$ proposed in \cite{FastSDD}. This recently proposed map is much faster than Laplacian Eigenmaps \cite{LaplacianEig} because it is based on optimization of a quadratic objective function under a linear constraint while in Laplacian Eigenmaps the optimization is of a quadratic objective function under a quadratic constraint. In addition to this, the solution for the fast manifold learning map can be obtained by solving a linear system of the form $\mathbf{Ax=b}$ where $\mathbf{A}$ happens to be a symmetric diagonally dominant (SDD) matrix. The solutions of such SDD linear systems can be computed very fast thereby leading to speedup involved with this Fast SDD Manifold Learning Map. \par  In order to be able to define the map $\mathbf{Z}(\mathbf{X},\sigma)$, we introduce three matrices: $\mathbf{L}(\mathbf{X},\sigma)$, $\mathbf{S}$, $\mathbf{\Gamma}$ that we now define.  
 The entries of graph laplacian matrix $\mathbf{L}(\mathbf{X},\sigma)_{n \times n}$  are defined using the Euclidean distance between the rows $i,j$ of $\mathbf{X}$ and a scalar $\sigma \in \mathbb{R}^{+}$as: 
 
 \begin{equation}\label{LEqn1}
\mathbf{L(\mathbf{X}},\sigma)_{ik}=\left\{
\begin{matrix} 
\sum_{k \ne i} e^{( - \frac{\Vert X_i - X_k\Vert^2}{ \sigma} )} & \mbox{if}\ i = k \\
-e^{\frac{- \Vert X_i - X_k\Vert^2}{ \sigma}} & \mbox{if}\ i \neq k
\end{matrix}\right\} \end{equation}
Note: The scalar $\sigma$ in here is also referred to as kernel bandwidth \cite{Ting}.

The matrix $\mathbf{S}_{n \times n}$ is given by:
\begin{equation}
\mathbf{S}_{ij}:=\left\{
\begin{matrix} 
-1 & \mbox{if}\ i \neq j \\
(n-1) & \mbox{if}\ i = j
\end{matrix}
\right\} \end{equation}

The matrix $\mathbf{\Gamma}_{n \times 2}$ is a matrix with very trivial requirement of all rows being distinct (i.e, differ by at least one entry) and for practical purposes we choose a $\mathbf{\Gamma}$ through sampling the entries $\mathbf{\Gamma}_{ij}$ from an i.i.d Normal distribution.

We finally define the map $\mathbf{Z}(\mathbf{X},\sigma)$ introduced in \protect{\cite{FastSDD}} using $\mathbf{A}(\sigma)=\mathbf{L}^{+}(\mathbf{X},\sigma)\mathbf{S\Gamma}$ as \begin{equation} \label{manifMap} 
\mathbf{Z}(\mathbf{X}, \sigma) = \frac{\mathbf{A}(\sigma)}{\Tr\left (\mathbf{\Gamma^TSA}(\sigma)\right)} 
\end{equation} 
where $\mathbf{L}^+(\mathbf{X},\sigma)$ denotes the Moore-Penrose pseudoinverse \protect{\cite{MoorePenrose}} of $\mathbf{L}(\mathbf{X},\sigma)$. Note that this map is a continuous function of the higher dimensional data set for any fixed bandwidth parameter; in other words, for every fixed $\sigma$, $\mathbf{Z}(\mathbf{X}(\epsilon), \sigma) \to \mathbf{Z}(\mathbf{X}, \sigma)$ as $\mathbf{X}(\epsilon) \to \mathbf{X}$. This is an immediate consequence of the continuity of the pseudoinverse $\mathbf{L^{+}}(\mathbf{X}(\epsilon))$ when the underlying matrices have constant rank \cite{rankLaplacianContinuity}, which is the case for Laplacian matrices.

\section{\texorpdfstring{$L^2$}{TEXT} energy for bandwidth selection and main result}In this section, we propose an objective function $\mathbf{E}(\mathbf{Z}(\sigma))$ that, on being minimized with respect to $\sigma$ gives the best bandwidth $\sigma_*$ at the optima for the purposes of manifold learning. We support this choice of function $\mathbf{E}(\mathbf{Z}(\sigma))$ through experimental results as well. In the rest of paper for ease of notation we refer to $\mathbf{L}(\mathbf{X},\sigma)$ as $\mathbf{L}(\sigma)$. \par The proposed energy function is the sum of squares of the sides of the projected ordered data set, i.e. \begin{equation} \mathbf{E}(\mathbf{Z}(\sigma))=\mathbf{E}(\sigma) = \sum_{i=1}^{n} \Vert \mathbf{Z_{i+1}} - \mathbf{Z_i} \Vert^2 + \Vert \mathbf{Z_n} - \mathbf{Z_1} \Vert ^2 \end{equation} 
This can be written in a more compact form as:   
\begin{equation}
\mathbf{E}(\mathbf{Z}(\sigma)) = Tr(\mathbf{Z^{T}}(\sigma)\mathbf{SZ}(\sigma))
\end{equation}Upon substituting \begin{equation}\mathbf{Z}(\sigma) = \frac{\mathbf{L}^{+}(\sigma)\mathbf{S\Gamma}}{tr\left(\mathbf{\Gamma^TSL}^{+}(\sigma)\mathbf{S\Gamma} \right)}\end{equation} in above expression of $\mathbf{E}(\mathbf{Z}(\sigma))$ we get 
\begin{equation}
\mathbf{E}(\mathbf{Z}(\sigma)) = \frac{Tr\left(\mathbf{\Gamma^TS(L^{+}(\sigma)^2S\Gamma}\right)}{\left[Tr\left(\mathbf{\Gamma^TSL}^{+}(\sigma)\mathbf{S\Gamma} \right)\right]^2} \end{equation} is essentially the $L^2$  based computation of perimeter of polygon formed by points in $\mathbf{Z}(\sigma)$ considered in a sequential order $\mathbf{Z}_1,\mathbf{Z}_2 \dots \mathbf{Z}_n, \mathbf{Z}_1$.

\subsection{\mbox{\bf  Main Result:}}\label{MR} If the result of the Fast SDD Manifold Learning Map \cite{FastSDD} given by $\mathbf{Z}(\mathbf{X},\sigma_*)$ at $\sigma=\sigma_*$ that minimizes $\mathbf{E}(\mathbf{X},\sigma_*)$ happens to form a polygon in $\mathbb{R}^2$ upon connecting the resultant points in the same order as in $\mathbf{X}$, then we infer that the original dataset in $\mathbb{R}^d$ is a topological circle. Conversely, if the result of the Fast SDD Manifold Learning Map \cite{FastSDD} at the $\sigma_*$ that minimizes $\mathbf{E}(\mathbf{X},\sigma) \equiv \mathbf{E}(\mathbf{Z}(\sigma)) $ does not form a polygon in $\mathbb{R}^2$ then we infer that the original dataset in $\mathbb{R}^d$ is not a topological circle.

\subsection{Immediate corollary/application of the main result:}
In the use-cases where we need to check whether the data came from a topolgical circle in high dimension (for example,
periodic data as in the traits of the mice experiments described in section 9, we can project $\mathbf{X}$ onto $\mathbf{Z(\sigma)}$ for values
of sigma, build the energy function $E$, and search for a mimizer of $E$. If the minimizer does not exist, we conclude that
the original data is not a part of a topological circle, which in return could imply non-existence of periodic patterns.

\subsection{Other work on parameter selection for manifold learning:}
The authors in \cite{ChenBuja} provide an approach to estimate $k$, the number of nearest neighbors to use in the construction of the graph Laplacian. This method evaluates a given $k$ with respect to the preservation of $k$- neighborhoods in the original data. However, it is not known how a method for estimating $k$ can be translated into a method for estimating $\sigma$ or vice versa (the two graph construction methods
exhibit different asymptotic behaviour precisely because they give rise to different ensembles of neighborhoods \cite{Ting}. The work  in \cite{nVasil} suggests an heuristic approach for bandwidth selection that utilizes kernel density estimation. This approach does not necessarily suggest a way to infer about the presence of topological circularity in $\mathbb{R}^d$ and also estimates a different bandwidth $\sigma_1, \sigma_2 \ldots \sigma_n$ one for each of the $n$ points in this formulation. $\\\\$ We now state results that go towards a proof of the main result of our paper stated in section \ref{MR}.

\section{Towards a proof of the main result: computation of derivatives of \texorpdfstring{$\mathbf{E}(\mathbf{Z}(\sigma))$}{TEXT}:} From this section onwards, we orient ourselves towards a proof of the main result 5.1. While a complete mathematical proof is still unknown to us at this moment, although will be highly desirable, we provide the proof by a mixture of theoretical computations, and experiments with data sets, both synthetic and real. Since the main theorem connects two seeming different quantities-minimizer of the $L^2$ energy $\mathbf{E}$, and the topological circularity of the projection $\mathbf{Z}$, we start first by showing theoretically that a minimizer of E indeed does exist in general cases of the original high dimensional dataset X. This is the subject matter of the sections 6 and 7. After proving the existence of the minimizer we test this on several data sets by plotting $\mathbf{Z(\sigma_\ast)}$ for $\sigma_\ast = \underset{\sigma}{\mathrm{ argmin }}\enspace\mathbf{E}(\sigma)$ and in each case find that $\mathbf{Z}(\sigma_\ast)$ indeed corresponds to a non self-intersecting polygon in two dimensions as we shall see in the diagrams towards the end of this paper in section 8.

 We then compute some derivatives with regards to our fast manifold learning map and proposed energy, that we would like to use in deriving some limit theorems in section 7. These limit theorems will explain part of the asymptotic behavior of $\mathbf{E}$ as a function of $\sigma$ which is crucial in proving the main result.
We now give some derivatives with regards to our fast manifold learning map and proposed energy that we would like to use in deriving some limiting theorems in rest of our paper.

\subsection{Limit based definition of \texorpdfstring{$\mathbf{L}^+(\sigma)$}{TEXT}}
In this subsection we give the derivative of $\frac{d \mathbf{L}^+(\sigma)}{d\sigma}$ which occurs in our map $\mathbf{Z}(\sigma)$ as it would later help us to define the derivative $\frac{d\mathbf{Z}(\sigma)}{d\sigma}$. Prior to that, we first state the standard limit based definition of Moore-Penrose pseudoinverse and also show the commutativity of pseudoinverses using their limit based definitions. For a small $\delta$
\begin{equation}
    \mathbf{L}^+(\sigma) = \lim_{\delta \to 0} (\mathbf{L}^2(\sigma) + \delta \mathbf{I})^{-1}L(\sigma) =  \lim_{\delta \to 0}
    \mathbf{L}(\sigma)(\mathbf{L}^2(\sigma) + \delta \mathbf{I})^{-1}
\end{equation}From this definition we have 
\begin{equation}
\mathbf{L}(\sigma)\mathbf{L}^+(\sigma) =  \lim_{\delta \to 0}  \mathbf{L}^2(\sigma)(\mathbf{L}^2(\sigma) + \delta \mathbf{I})^{-1} = (\mathbf{L}^2(\sigma) + \delta \mathbf{I})^{-1}\mathbf{L}^2(\sigma) = \mathbf{L}^+(\sigma)\mathbf{L}(\sigma)
\end{equation}

\subsection{Derivative of \texorpdfstring{$\mathbf{L}^+(\sigma)$}{TEXT}}
\begin{lemma} The derivative of $\mathbf{L}^+(\sigma)$ with respect to $\sigma$ where  $\mathbf{L}(\sigma)$ is defined for any $\mathbf{X} \in \mathbb{R}^d$ as   \begin{equation}\label{LEqn}
\mathbf{L(\mathbf{X}},\sigma)_{ik}=\left\{
\begin{matrix} 
\sum_{k \ne i} e^{( - \frac{\Vert X_i - X_k\Vert^2}{ \sigma} )} & \mbox{if}\ i = k \\
-e^{\frac{- \Vert X_i - X_k\Vert^2}{ \sigma}} & \mbox{if}\ i \neq k
\end{matrix}
\right\} \end{equation}is given by
\begin{equation}
\frac{d\mathbf{L}^+(\sigma)}{d\sigma} = -(\mathbf{L}^+(\sigma))^2\frac{d\mathbf{L}}{d\sigma}
\end{equation}
\end{lemma} 

\begin{proof}
We know from the properties of Moore-Penrose pseudoinverse that $$
\mathbf{L}(\sigma)\mathbf{L}^{+}(\sigma)\mathbf{L}(\sigma) = \mathbf{L}(\sigma)
$$The rank of a graph Laplacian matrix is $(n-1)$ as it has it's smallest eigenvalue as zero with multiplicity 1. \par Now the derivative of a real valued pseudoinverse matrix $\mathbf{L}^+(\sigma)$which has constant rank at a point $\sigma$ may be calculated in terms of derivative of $\mathbf{L}(\sigma)$ as given in equation 4.12 of \cite{GolubDer} as:
\begin{align}
\frac{d}{d\sigma}\mathbf{L}^+(\sigma) &= -\mathbf{L}^+(\sigma) \left( \frac{d}{d\sigma} \mathbf{A}(\sigma)\right)\mathbf{L}^+(\sigma) + \mathbf{L}^+(\sigma)\mathbf{L}^+(\sigma)^\mathbf{T}\left( \frac{d}{d\sigma} \mathbf{A}(\sigma)^\mathbf{T}\right)\left(\mathbf{I}-\mathbf{A}(\sigma)\mathbf{A}^+(\sigma)^\mathbf{T}\right)\\ &+\left(\mathbf{I}-\mathbf{A}^+(\sigma)^\mathbf{TA}(\sigma)\right) \left(\frac{d}{d\sigma} \mathbf{A}(\sigma)^\mathbf{T}\right)\mathbf{A}^+(\sigma)^\mathbf{T}\mathbf{A}^+(\sigma)
\end{align} 
Now as $\mathbf{L}(\sigma)$ is symmetric we have \begin{equation}\label{eqn:dpids}
\frac{d\mathbf{L}^+(\sigma)}{d \sigma }= \frac{d\mathbf{L}(\sigma)}{d\sigma}  ( \mathbf{I} - 2\mathbf{L}(\sigma)\mathbf{L}^+(\sigma) ) ( \mathbf{L}^+(\sigma) )^2 =  -(\mathbf{L}^+(\sigma))^2\frac{d\mathbf{L}(\sigma)}{d\sigma}
\end{equation}
\end{proof}

\subsection{Derivatives of \texorpdfstring{$\mathbf{Z}(\sigma)$}{TEXT}}The first derivative of our manifold learning map with respect to $\sigma$ is given by
\begin{equation}
\frac{d\mathbf{Z}(\sigma)}{d\sigma} = \frac{\Tr\left(\mathbf{\Gamma^T S A}(\sigma)\right)\frac{d\mathbf{A}\left(\sigma\right)}{d\sigma} - \mathbf{A}\left(\sigma\right) \Tr\left(\mathbf{\Gamma^TS}\frac{d\mathbf{A}(\sigma)}{d\sigma} \right)}{\Tr\left(\mathbf{\Gamma^TSA}(\sigma)\right)^2}
\end{equation}
Using \ref{eqn:dpids} we have $\frac{d\mathbf{A}(\sigma)}{d\sigma} = -\left(\mathbf{L}^+(\sigma)\right)^2\frac{d\mathbf{L}(\sigma)}{d\sigma}\mathbf{S\Gamma} $ that we substitute above to get
\begin{equation}\label{eq:dzds}
\frac{d\mathbf{Z}(\sigma)}{d\sigma} = \frac{tr\left(\mathbf{\Gamma^T S A}(\sigma)\right)[-(\mathbf{L}^{+}\left(\sigma)\right)^2 \frac{d\mathbf{L}}{d\sigma}\mathbf{S\Gamma}]+ \mathbf{A}(\sigma)tr\left(\mathbf{\Gamma^T S} \mathbf{L}^{+}(\sigma)\right)^2 \frac{d\mathbf{L}}{d\sigma}\mathbf{S\Gamma}}{tr(\mathbf{\Gamma^T S A}(\sigma))^2}
\end{equation}

\subsection{Derivatives of \texorpdfstring{$\mathbf{E}(\mathbf{Z}(\sigma))$}{TEXT}}
The first derivative of our $L^2$  based energy function $\mathbf{E}(\mathbf{Z}(\sigma))$ with respect to $\sigma$ is given by \begin{equation}
\frac{d\mathbf{E}(\mathbf{Z}(\sigma))}{d\sigma} = 2 Tr\left[ \mathbf{SZ}(\sigma)\frac{d\mathbf{Z}}{d\sigma}^\mathbf{T}\right]
\end{equation} The second derivative of $\mathbf{E}(\mathbf{Z}(\sigma))$ is given by \begin{equation}
\frac{d^2\mathbf{E}(\mathbf{Z}(\sigma))}{d\sigma^2} = 2Tr\left [  \mathbf{S} \frac{d\mathbf{Z}}{d\sigma} \frac{d\mathbf{Z}}{d\sigma}^\mathbf{T} + \mathbf{SZ}(\sigma) \frac{d^2\mathbf{Z}(\sigma)}{d\sigma^2}^\mathbf{T} \right ]
\end{equation}

\section{Towards a proof of the main result: Asymptotic behavior of \texorpdfstring{$\mathbf{E}(\mathbf{Z}(\sigma))$}{TEXT}} In this section, we show theoretically the limits of $\mathbf{E}$ and its certain derivatives as $\sigma$ approaches $\infty$ are finite. In the next section, we will present several experimental proofs that the limits of $\mathbf{E}$ and its certain derivatives are finite also when $\sigma$ approaches zero. These two results together prove that: $\mathbf{E}$ is indeed bounded continuous function on $(0, \infty)$, and hence has a global minimum.

\begin{lemma}
The following two properties hold true for matrix $\mathbf{S}$:
\begin{enumerate}[i)]
    \item
    $\mathbf{S}^{+} = \frac{1}{n^2} \mathbf{S}$
    \item 
    $ \lim_{\sigma \to \infty} \mathbf{L}^{+}(\sigma) = \mathbf{S}^{+}$
\end{enumerate}
\end{lemma}

\begin{proof} 
\begin{enumerate}[i)]
\item It follows from using the fact $\mathbf{S^T}=\mathbf{S}$  that $$\mathbf{SS^T}=\mathbf{S^TS}=\mathbf{S}^2=n\mathbf{S}$$ This implies
$$\mathbf{S  S}^+   \mathbf{S}= \mathbf{S} (1/n)^2 \mathbf{S  S} =  (1/n)^2 \mathbf{S}^3 =  (1/n)^2 n^2 \mathbf{S} = \mathbf{S}$$
Now using the uniqueness of $\mathbf{S}^+$ from the definition of Moore-Penrose pseudoinverse we conclude that $\mathbf{S}^+ = (1/n)^2\mathbf{S}$.
\item
On substituting the simple limit $\lim_{\sigma \to \infty}e^\frac{- \Vert X_i - X_j \Vert ^2}{\sigma} = 1 $ in the definition of $\mathbf{L}$ in eqn. \ref{LEqn} we get $\lim_{\sigma \to \infty} \mathbf{L = S}$ and therefore we have $\lim_{\sigma \to \infty}\mathbf{L}^+(\sigma) = \mathbf{S}^+$ . 
\end{enumerate}
\end{proof}

\begin{lemma} 
 The following limit over the operator norm of $\mathbf{L}^+(\sigma)$ holds true: $$\lim_{\sigma \to 0}\Vert \mathbf{L}^+(\sigma) \Vert_{op} = \infty$$
\end{lemma}
\begin{proof}
\begin{align}
\Vert \mathbf{L}(\sigma)_{op}&= \Vert \mathbf{L}(\sigma) \mathbf{L}^+(\sigma) \mathbf{L}(\sigma) \Vert_{op}\\ \nonumber
&\leq \Vert \mathbf{L}(\sigma) \Vert_{op}\Vert \mathbf{L}^+(\sigma) \Vert_{op} \Vert \mathbf{L}(\sigma) \Vert_{op}\\ \nonumber
&= \Vert \mathbf{L}(\sigma) \Vert_{op}^{2} \Vert \mathbf{L}^{+}(\sigma) \Vert_{op}\\ \nonumber
\end{align}
Therefore, $$ \Vert \mathbf{L}^{+}(\sigma) \Vert_{op} \geq \left ( \Vert \mathbf{L}(\sigma) \Vert_{op}\right)^{-1}$$
For $r_2 = \min_{k \neq l} \Vert \mathbf{X}_k - \mathbf{X}_l\Vert^2 , \mathbf{A}_2 > 0$ upon substituting $\Vert \mathbf{L}(\sigma) \Vert_{op} \leq \mathbf{A}_2 e^{\frac{-r_2}{\sigma}} $  above we get
\begin{equation}
\Vert \mathbf{L}^+(\sigma)\Vert_{op} \geq \frac{1}{\mathbf{A}_2}e^\frac{r^2}{\sigma}
\end{equation}

On computing limit of this inequality we get $\lim_{\sigma \to 0}\Vert \mathbf{L}^+(\sigma) \Vert_{op} = \infty$.
\end{proof}
\begin{theorem} 
The following limits hold true for manifold learning map $\mathbf{Z}(\sigma)$:

\begin{enumerate}[i)]

\item
When $n \to \infty$ \text{we have} $\lim_{\sigma \to \infty} \mathbf{Z}(\sigma) =   \frac{\nu \mathbf{\Gamma}}{Tr(\mathbf{\Gamma^T S \Gamma})}$

\item 
$\lim_{\sigma \to \infty} \frac{d\mathbf{Z}}{d\sigma} = 0$

\item
$\lim_{\sigma \to \infty} \frac{d^2\mathbf{Z}}{d\sigma^2} = 0$

\end{enumerate}

\end{theorem}
\begin{proof}
\begin{enumerate}[i)]
\item
\begin{equation}
    \mathbf{Z}(\sigma) = \frac{\nu \mathbf{L}^{+}(\sigma)\mathbf{S \Gamma}}{Tr(\mathbf{\Gamma^T S L}^{+}(\sigma))\mathbf{S\Gamma}}
\end{equation}
Upon substituting $\lim_{\sigma \to \infty} \mathbf{L}^+ = \mathbf{S}^+$ and  $\mathbf{S}^+ = \frac{1}{n^2}\mathbf{S}$ in above expression we get
\begin{align}
\lim_{\sigma \to \infty} \mathbf{Z}(\sigma)&=\frac{\nu \mathbf{S}^{+}(\sigma)\mathbf{S\Gamma}}{Tr(\mathbf{\Gamma^TSS}^{+}(\sigma)\mathbf{S\Gamma})}\\
      &= \frac { (\nu/n^2) \mathbf{S}^2 \mathbf{\Gamma}} {(1/n^2)Tr (\mathbf{\Gamma^T S}^3 \mathbf{\Gamma}) } \nonumber\\ 
\end{align}
We substitute $\mathbf{S}^2 = n\mathbf{S}$ in the above expression and as well using the same we have $$\mathbf{S}^3 = \mathbf{S}^2 \mathbf{S} = n\mathbf{S}^2 = n^2 \mathbf{S}$$ which we also substitute above 
\begin{align}\label{ZLimFinal}
&=\frac{\nu (\mathbf{S}/n)\mathbf{\Gamma}}{ Tr(\mathbf{\Gamma^TS\Gamma})} 
\end{align}
As $n \to \infty$, we have $\frac{\mathbf{S}}{n} = \mathbf{I}$ which follows directly from the definition of $\mathbf{S}$ as $\mathbf{S}/n = \begin{bmatrix} 1-\frac{1}{n} & \frac{-1}{n} \\ \frac{-1}{n} & 1-\frac{1}{n}  \end{bmatrix}$. We substitute this in \ref{ZLimFinal} to obtain the required limit.

\item
We have $\lim_{\sigma \to \infty} \mathbf{L}^+(\sigma) = \mathbf{S}^+$ and therefore we have
\begin{equation}
\lim_{\sigma \to \infty} \mathbf{A}(\sigma) = \mathbf{L}^{+}\mathbf{S\Gamma} = \mathbf{\Gamma} \end{equation}
We substitute this in limit of eqn. \ref{eq:dzds} of $\frac{d\mathbf{Z}(\sigma)}{d\sigma}$ to get

\begin{equation}\label{dzdss}
\lim_{\sigma \to \infty}\frac{d\mathbf{Z}(\sigma)}{d\sigma} = \lim_{\sigma \to \infty}\frac{tr(\mathbf{\Gamma^T S \Gamma})[-(\mathbf{S}^{+})^2 \frac{d\mathbf{L}}{d\sigma}\mathbf{S\Gamma}]+\mathbf{\Gamma} tr(\mathbf{\Gamma^T S} (\mathbf{S}^{+})^2 \frac{d\mathbf{L}}{d\sigma}\mathbf{S\Gamma})}{tr(\mathbf{\Gamma^T S \Gamma})^2}
\end{equation}
the derivative of the off-diagonal term $\mathbf{L}_{ij}(\sigma)$ is \begin{equation}\frac{d\mathbf{L}_{ij}(\sigma)}{d\sigma}=\frac{ -\Vert X_i - X_j \Vert ^2 }{\sigma ^ 2   e^{-\Vert X_i - X_j \Vert ^2) / \sigma}  } \end{equation} and the derivative of the diagonal term $\mathbf{L}_{ii}(\sigma)$ is \begin{equation}\frac{d\mathbf{L}_{ii}(\sigma)}{d\sigma}=\sum_{j \neq i} \left[\frac{ -\Vert X_i - X_j \Vert ^2 }{\sigma ^ 2   e^{-\Vert X_i - X_j \Vert ^2) / \sigma}  }   \right]\end{equation} From this we have \begin{equation}\label{dldszero}
\lim_{\sigma \to \infty}=\frac{d\mathbf{L}(\sigma)}{d\sigma} = 0
\end{equation} 
We substitute this in equation \ref{dzdss} to get our required result $$\lim_{\sigma \to \infty} \frac{d\mathbf{Z}(\sigma)}{d\sigma} = 0$$

\item
We refer to the numerator of the expression of $\frac{d\mathbf{Z}(\sigma)}{d\sigma}$ with $\mathbf{N}(\sigma)$ as $$\mathbf{N}(\sigma)= \Tr\left(\mathbf{\Gamma^T S A}(\sigma)\right)\frac{d\mathbf{A}\left(\sigma\right)}{d\sigma} - \mathbf{A}\left(\sigma\right) \Tr\left(\mathbf{\Gamma^TS}\frac{d\mathbf{A}(\sigma)}{d\sigma} \right)$$
and we refer to the denominator of the expression of $\frac{d\mathbf{Z}(\sigma)}{d\sigma}$ with $\mathbf{D}(\sigma)$ as $$\mathbf{D}(\sigma) = \Tr\left(\mathbf{\Gamma^TSA}(\sigma)\right)^2$$
We have 
\begin{equation}\label{nonzero}
\lim_{\sigma \to \infty}\mathbf{D}(\sigma) = \Tr(\mathbf{\Gamma^{T} S} \frac{1}{n^2}\mathbf{S}^2\Gamma) \neq \mathbf{0}\end{equation} 
The derivative of $\mathbf{D}(\sigma)$ is given below
 $$\frac{d\mathbf{D}(\sigma)}{d\sigma}=2\Tr\left(\mathbf{\Gamma^{T}SA}(\sigma)\right)\Tr\left(\mathbf{\Gamma^{T}S}\frac{d\mathbf{A}(\sigma)}{d\sigma}\right)
 $$
From equation \ref{dldszero}, we have $\lim_{\sigma \to \infty}\frac{d\mathbf{L}(\sigma)}{d\sigma} = 0$ and therefore $\lim_{\sigma \to \infty}\mathbf{A}(\sigma)=0$. Based on these limits and the equation above we have \begin{equation}\label{zero1}
\lim_{\sigma \to \infty} \frac{d\mathbf{D}(\sigma)}{d\sigma}=0\end{equation} The expression for $\frac{d\mathbf{N}(\sigma)}{d\sigma}$ is a sum of terms containg $\frac{d\mathbf{A}(\sigma)}{d\sigma}$, $\Tr\left(\mathbf{\Gamma^TS}\frac{d\mathbf{A}(\sigma)}{d\sigma}\right)$ and $\frac{d^2\left(\mathbf{A}(\sigma)\right)}{d\sigma^2}$, each of which tend to $0$ as $\sigma \to \infty$ and therefore \begin{equation} \label{zero2} \lim_{\sigma \to \infty}\frac{d\mathbf{N}(\sigma)}{d\sigma} = 0 \end{equation}
We have \begin{equation}
\frac{d^2\mathbf{Z}(\sigma)}{d\sigma^2} = \frac{\mathbf{D}(\sigma)\frac{d\mathbf{N}(\sigma)}{d\sigma} - \mathbf{N}(\sigma)\frac{d\mathbf{D}(\sigma)}{d\sigma}}{\mathbf{D}^2(\sigma)}
\end{equation}
On substituting \ref{nonzero}, \ref{zero1} and \ref{zero2} above we prove $\lim_{\sigma \to \infty}\frac{d^2\mathbf{Z}(\sigma)}{d\sigma^2} = 0$.
\end{enumerate}
\end{proof}

\begin{lemma} The following limits hold true for energy $\mathbf{E}(\mathbf{Z}(\sigma))$:
\begin{enumerate}[i)]
\item $\lim_{\sigma \to \infty} \mathbf{E}(\sigma)$ is finite.
\item $
\lim_{\sigma \to \infty} \frac{d\mathbf{E}(\mathbf{Z}(\sigma))}{d\sigma} = 0
$
\item
$\lim_{\sigma \to \infty} \frac{d^2\mathbf{E}(\mathbf{Z}(\sigma))}{d\sigma^2} = 0
$

\end{enumerate}
\end{lemma}

\begin{proof}
\begin{enumerate}[i)]
\item This is immediate from the proof of Theorem 1, that expresses a limit for $\mathbf{Z}(\sigma)$ as $\sigma \to \infty$, and the expression of $\mathbf{E}(\sigma)$ in section 5.
\item The first derivative of $\mathbf{E}(\mathbf{Z}(\sigma))$ is given by \begin{equation}
\frac{d\mathbf{E}(\sigma)}{d\sigma} = 2 Tr\left[ \mathbf{SZ}(\sigma)\frac{d\mathbf{Z}}{d\sigma}^\mathbf{T}\right]
\end{equation}

As $\lim_{\sigma \to 0} \frac{d\mathbf{Z}}{d\sigma} = 0$ and $\lim_{\sigma \to \infty} \mathbf{Z}(\sigma) =   \frac{\nu \mathbf{\Gamma}}{Tr(\mathbf{\Gamma^T S \Gamma})}$ we prove the result.

\item The second derivative of $\mathbf{E}(\mathbf{Z}(\sigma))$ is given by \begin{equation}
\frac{d^2\mathbf{E}(\sigma)}{d\sigma^2} = 2Tr\left [ \mathbf{ S} \frac{d\mathbf{Z}}{d\sigma} \frac{d\mathbf{Z}}{d\sigma}^T + \mathbf{SZ}(\sigma) \frac{d^2\mathbf{Z}(\sigma)}{d\sigma^2}^\mathbf{T} \right ]
\end{equation} 
As $\lim_{\sigma \to 0} \frac{d\mathbf{Z}}{d\sigma} = 0$ and  $\lim_{\sigma \to 0} \frac{d^2\mathbf{Z}}{d\sigma} = 0$ and $\lim_{\sigma \to \infty} \mathbf{Z}(\sigma) =   \frac{\nu \mathbf{\Gamma}}{Tr(\mathbf{\Gamma^T S \Gamma})}$ we prove the result.
\end{enumerate}
\end{proof}
Setting aside the experimental part of the proof for section 8, we write down two algorithms, one to compute $\sigma_{\ast} = \underset{\sigma}{\mathrm{ arg \enskip min }}\enspace\mathbf{E}(\sigma)$ for a fixed dataset $\mathbf{X}$, and the next one is to test the existence of a topological circularity in $\mathbf{X}$. 
\begin{algorithm}[ht]
\caption{Bandwidth Estimation Algorithm}\label{euclid}
\begin{algorithmic}[1]
\Procedure{Bandwidth Estimation}{}
\State $\text{Construct }
 \mathbf{L(\sigma)} \text{ as }
\mathbf{L}(\mathbf{X},\sigma)_{ij}=\left\{
\begin{matrix} 
\sum_{j \neq i, j=1}^{n}e^{\frac{-\Vert X_i - X_j \Vert ^2}{ \sigma}} & \mbox{if}\ i = j \\
-e^{\frac{-\Vert X_i - X_j \Vert ^2}{ \sigma}} & \mbox{if}\ i \neq j
\end{matrix}
\right\} $
\State $ \textit{Minimization Phase:} \text{Apply gradient descent to find a local minima }\sigma_{l} \text{ of }$  $$\mathbf{E}(\mathbf{Z}(\sigma)) = Tr(\mathbf{Z}(\sigma)\mathbf{S}\mathbf{Z^T}(\sigma))$$ 
\enskip using its derivative in equation 17, which further uses the derivatives in equation 16, \enskip 27 and 28.
\State $\textit{Tunneling Phase: }$ Minimize $  \frac{\mathbf{E}(\mathbf{Z}(\sigma)) - \mathbf{E}(\mathbf{Z}(\sigma_l))}{(\sigma - \sigma_l)^\lambda}$, \text{ where }$\lambda$ \protect{\text{ is chosen as in \protect{\cite{tunneling3}}}} to get $\sigma_{\ast}$
\State Return $\sigma_{\ast}$
\EndProcedure
\end{algorithmic}
\end{algorithm}

\begin{algorithm}
\caption{Hypothesis testing for topological circularity}\label{euclidTwo}
\begin{algorithmic}[1]
\Statex \text{Hypothesis }$H_0$\text{: There is a circular structure in } $\mathbf{X}$.
\If{$\sigma_{\ast} = \underset{\sigma}{\mathrm{ arg \enskip min }}\enspace\mathbf{E}(\sigma)$ \text{ does not exist}} 
\State reject  $H_0$
\EndIf
\If{$\sigma_{\ast} = \underset{\sigma}{\mathrm{ arg \enskip min }}\enspace\mathbf{E}(\sigma)$ \text{ exists}}
\State plot $\mathbf{Z}(\sigma_{\ast})$, and join its points in the order of $n$ points in $\mathbf{X}$.
\EndIf
 \If{ $\mathbf{Z}(\sigma_{\ast})$ does not form not a non-self-intersecting polygon} 
 \State reject $H_0$ 
\EndIf
\State \mbox{else}
\If{ $\mathbf{Z}(\sigma_{\ast})$ forms a non-self-intersecting polygon}
\State accept $H_0$
\EndIf
\end{algorithmic}
\end{algorithm}

\section{ Experiments and conclusion of proof of the main result:}

As hinted back in section 7, this will be the concluding section for a proof of the main result where several experiments will guarantee that the minimizer of E corresponds to a non-self intersecting polygon Z(sigma) in 2D.
\subsection{Experiments with synthetic data:}
In this subsection we now describe two experiments that we have conducted with our proposed $\mathbf{E}(\mathbf{Z}(\sigma))$ functions on synthetic data lying on a 
\begin{enumerate}[a)]
\item Circle
\item Toroidal helix
\end{enumerate}

\begin{figure}[!ht] 
  \begin{subfigure}[b]{0.5\linewidth}
    \centering
    \includegraphics[width=0.75\linewidth]{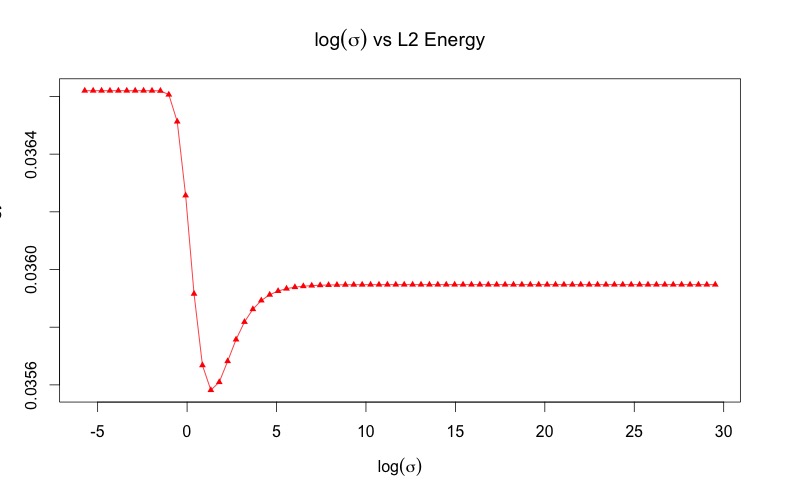} 
    \caption{$\protect X = \text{circle }, log(\sigma)$ Vs. $\protect L^2$ Energy,\\ $\protect n=4$} 
    \label{fig7:a} 
    \vspace{4ex}
  \end{subfigure}
  \begin{subfigure}[b]{0.5\linewidth}
    \centering
    \includegraphics[width=0.75\linewidth]{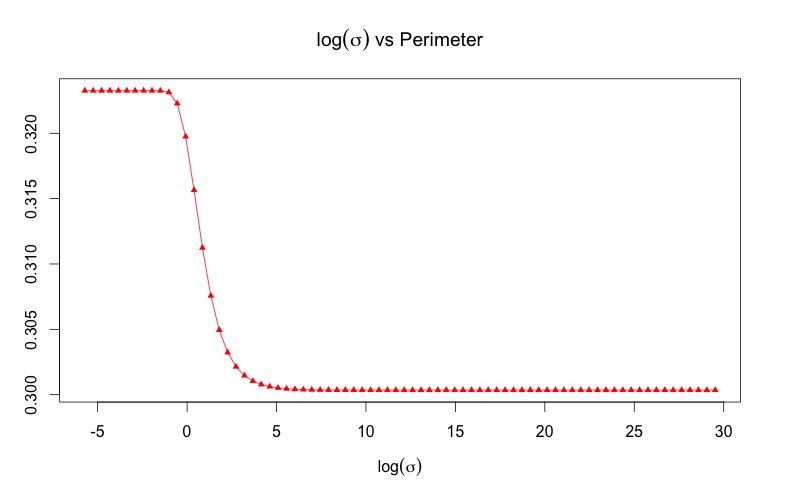} 
    \caption{$\protect X = \text{circle }, log(\sigma)$ Vs. $\protect L^1$ Perimeter,\\ \centering $\protect n=4$} 
    \label{fig7:b} 
    \vspace{4ex}
  \end{subfigure} 
  \begin{subfigure}[b]{0.5\linewidth}
    \centering
    \includegraphics[width=0.75\linewidth]{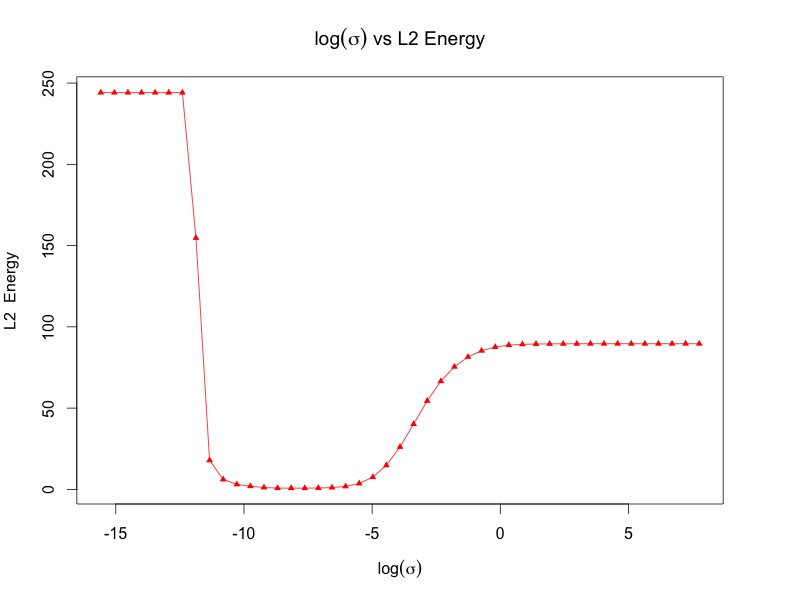} 
    \caption{$\protect X = \text{circle }, log(\sigma)$ Vs. $\protect L^2$ Energy,\\ $\protect n=750$} 
    \label{fig7:c} 
  \end{subfigure}
  \begin{subfigure}[b]{0.5\linewidth}
    \centering
    \includegraphics[width=0.75\linewidth]{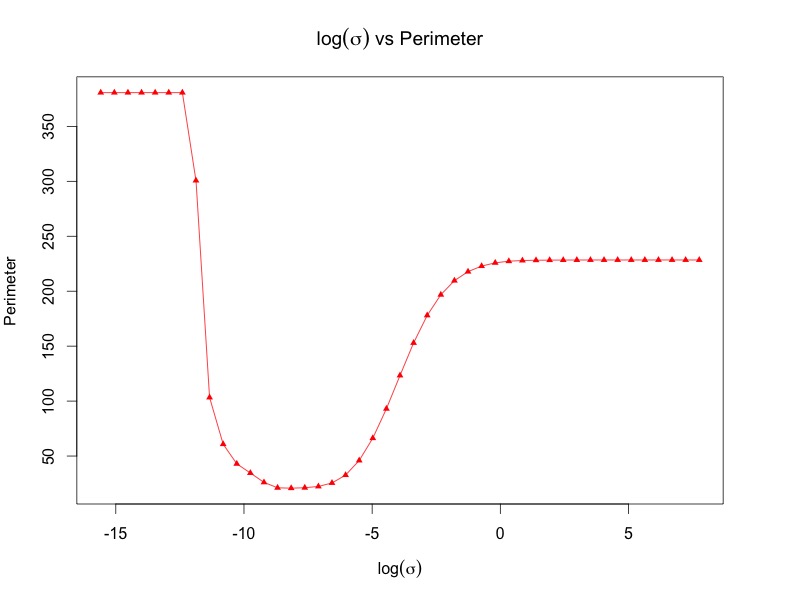} 
    \caption{$\protect X = \text{circle },log(\sigma)$ Vs. $\protect L^1$ Perimeter, \\\centering $\protect n=750$} 
    \label{fig7:d} 
  \end{subfigure} 
  \caption{\protect{$\protect log(\sigma)$ Vs. $\protect L^1$ and $\protect L^2$} energy and perimeter of fast \protect{\cite{FastSDD}} manifold learning on unit-circle.}
  \label{fig7} 
\end{figure}

\begin{figure}[!ht]   
  \begin{subfigure}[b]{0.5\linewidth}
    \centering
    \includegraphics[width=0.85\linewidth]{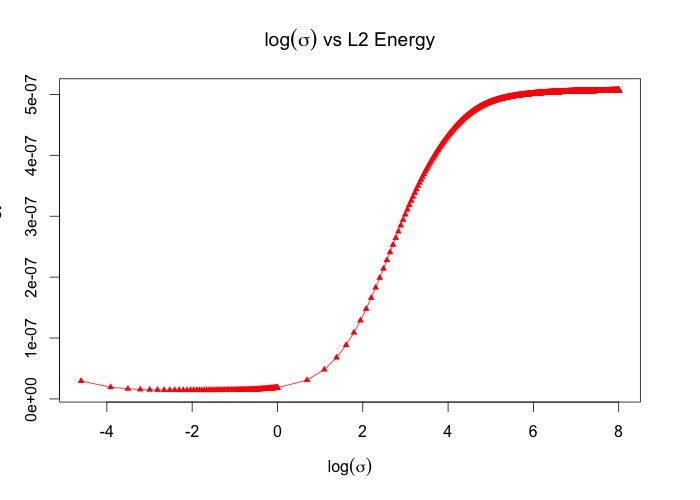} 
    \caption{$\protect X = \text{toroidal helix},log(\sigma)$ Vs. $\protect L^2$ Energy, \\\centering $\protect n=750$} 
    \label{fig8:a} 
  \end{subfigure}
  \begin{subfigure}[b]{0.5\linewidth}
    \centering
    \includegraphics[width=0.85\linewidth]{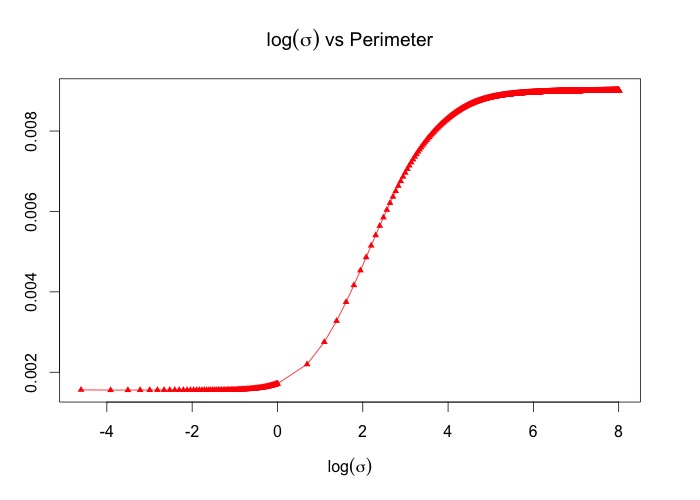} 
    \caption{$\protect X = \text{toroidal helix}, log(\sigma)$ Vs. $\protect L^1$ Perimeter, \centering$\protect n=750$} 
    \label{fig8:b} 
  \end{subfigure} 
  \caption{$\protect log(\sigma)$ Vs. $\protect L^1$ and $\protect L^2$ energy and perimeter of fast \protect{\cite{FastSDD}} manifold learning on a toroidal helix.}
  \label{fig8} 
\end{figure}
\begin{figure}
\begin{multicols}{3}
    \includegraphics[width=\linewidth]{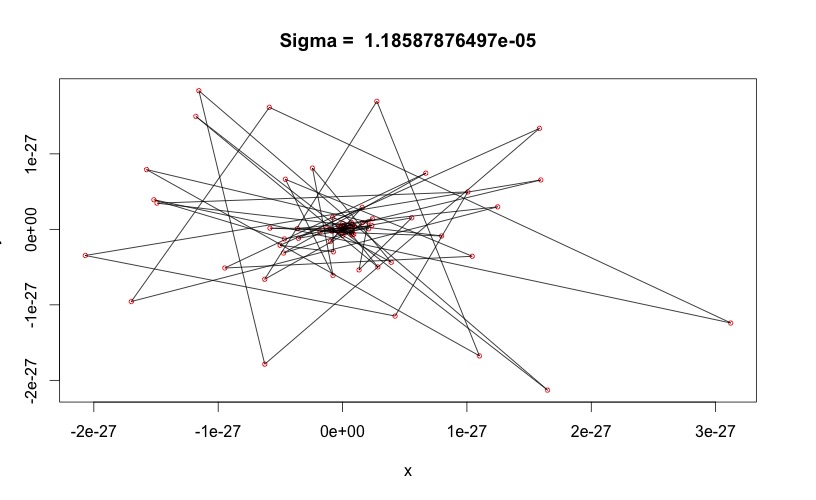}\par \caption{$\protect \sigma = 1.18 \times 10^{-5}$(Small $\protect \sigma$)}
    \includegraphics[width=\linewidth]{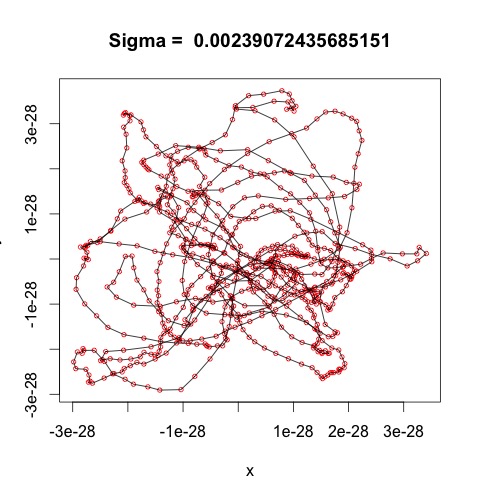}\par\caption{ $\protect \sigma = 0.00239$ (Small $\protect\sigma$)}
    \includegraphics[width=\linewidth]{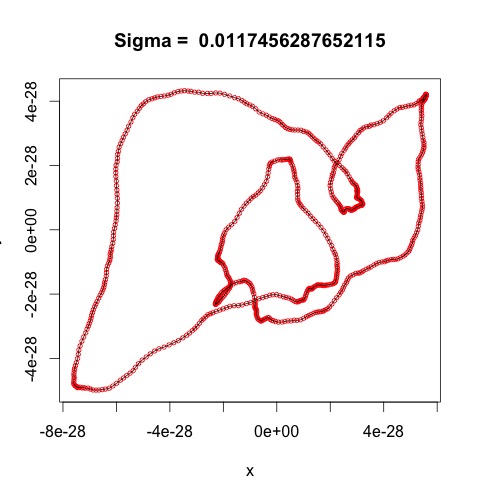}\par\caption{$\protect \sigma = 0.01174$ (Small $\protect\sigma$)}
\end{multicols}
\begin{multicols}{3}
    \includegraphics[width=\linewidth]{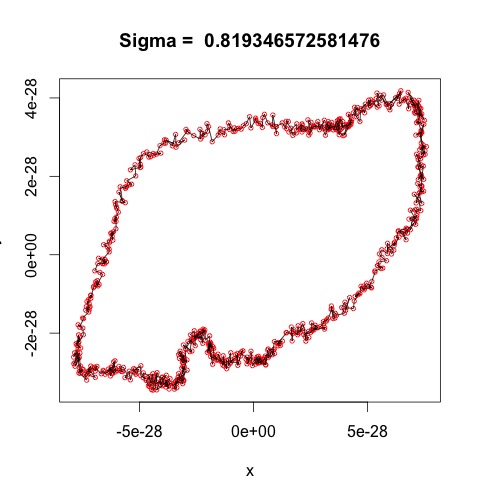}\par\caption{$\protect \sigma = 0.8193$ ($\protect\sigma$ around $\sigma_*$)}
    \includegraphics[width=\linewidth]{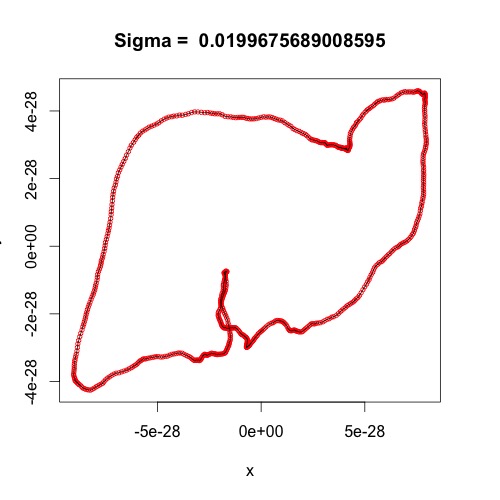}\par\caption{$\protect\sigma = 0.0199$ ($\protect\sigma$ around $\protect\sigma_*$)}
    \includegraphics[width=\linewidth]{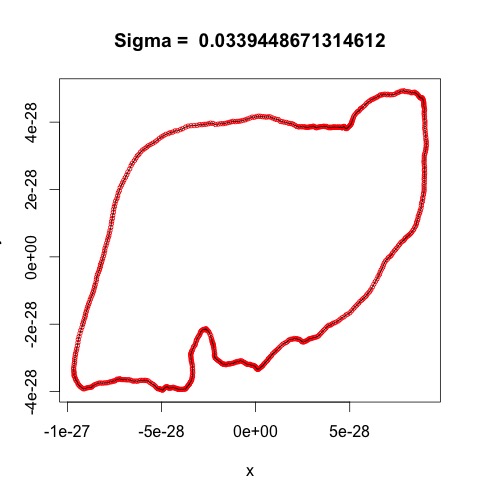}\par\caption{$\protect\sigma = 0.0339$ ($\protect\sigma$ around $\protect\sigma_*$)}
\end{multicols}
\begin{multicols}{3}
    \includegraphics[width=\linewidth]{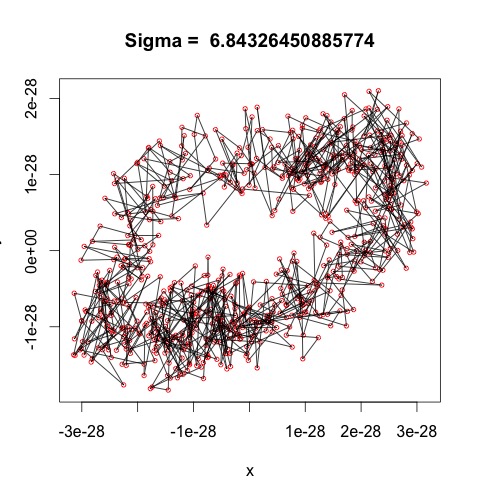}\par\caption{$\protect\sigma = 6.8432$ (large $\protect\sigma$)}
    \includegraphics[width=\linewidth]{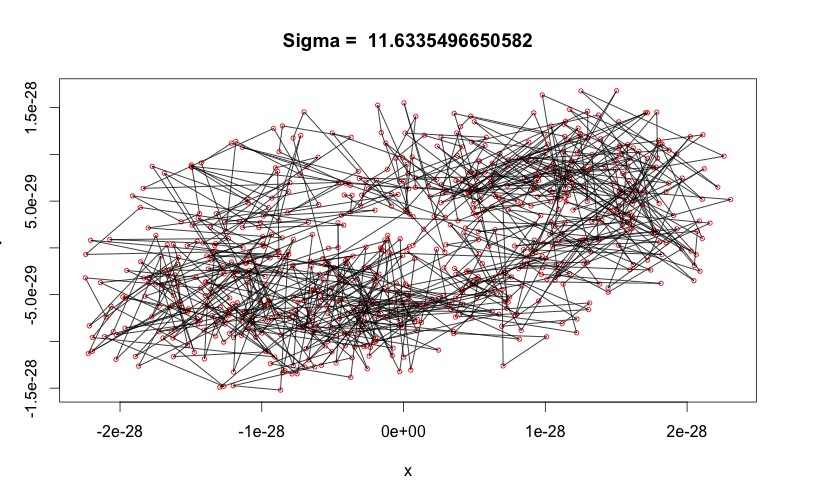}\par\caption{$\protect\sigma = 11.6335$ (large $\protect\sigma$)}
    \includegraphics[width=\linewidth]{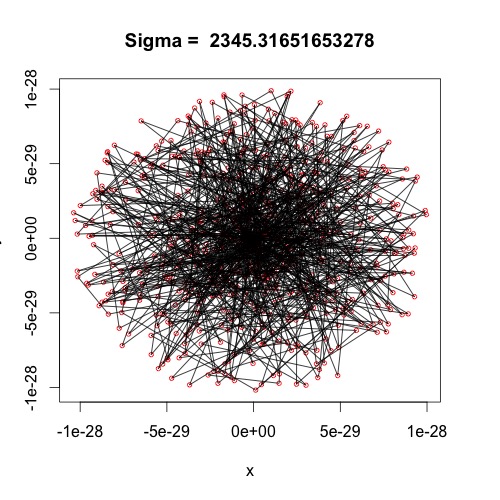}\par\caption{$\protect\sigma = 2345.31$ (large $\protect\sigma$)}
    
\end{multicols}
\centering\caption{$\protect \mathbf{Z(\mathbf{X},\sigma)} \text{ for } \mathbf{X} = $ toroidal helix, for different choices of $\protect \sigma$ }
\end{figure}

Within Figure 1, in subfigure (a) we provide a graph of $L^2$ energy, $\mathbf{E}(\mathbf{Z}(\sigma))$ (y-axis) with respect to varying $\sigma$ (x-axis) for a very small $n = 4$ points on a unit-circle. By perimeter, $\mathbf{P}(\mathbf{Z}(\sigma))$ we refer to equation 4 computed with $L^{1}$ instead of $L^{2}$ norms in the first and second summands. The subfigure (b) in Figure 1 provides the graph of $\mathbf{P}(\mathbf{Z}(\sigma))$ for the same $n=4$ points on a unit-circle. For $n=750$ points on a unit-circle we provide graphs of $\mathbf{E}(\mathbf{Z}(\sigma))$ and $\mathbf{P}(\mathbf{Z}(\sigma))$ in subfigures (c) and (d) of Figure 1. Within Figure 2 in subfigures (a) and (b) we provide $\mathbf{E}(\mathbf{Z}(\sigma))$ and $\mathbf{P}(\mathbf{Z}(\sigma))$ for $n=750$ points lying on a toroidal helix. By $\sigma_*$ we refer to the minimizer of $\mathbf{E}(\mathbf{Z}(\sigma))$. We can obtain $\sigma_*$ through gradient descent followed by a tunneling phase as detailed in Algorithm 1 of this paper.  Gradient descent is a popular algorithm for finding a local minima while tunneling is a novel way to find a local minima better than the one obtained through gradient descent. Hence, one could also apply multiple iterations of gradient descent followed by tunneling where each iteration consists of one followed by another. \par In Figures 3, 4 and 5, for $\mathbf{X}$ lying on a toroidal helix, we provide $\mathbf{Z}(\sigma)$ for small choices of $\sigma$ away from $\sigma_*$. The points are in red. We connect the points in an order $\mathbf{Z}_1, \mathbf{Z}_2 \ldots \mathbf{Z}_n$ with straight lines in black. We see that they do not form polygons when the choice of $\sigma$ is away from $\sigma_*$. In Figures 6, 7 and 8 we show corresponding geometries obtained for choices of $\sigma$ around optimal $\sigma_*$. We show especially in Figure 8 that as $\sigma \to \sigma_*$, we obtain a polygon, as there are no intersections. In Figures 9, 10 and 11 we show the geometries obtained for choices of large $\sigma$ away from $\sigma_*$ again do not form a polygon.\par We show in the sub-figures of 8(a) and 8(b) that the value of $log(\sigma)$ which gives the least $L^2$ energy $\mathbf{E}(\mathbf{Z}(\sigma))$ and $L^1$ perimeter $\mathbf{P}(\mathbf{Z}(\sigma))$ corresponds to the $\sigma$ that happens to be a non self-intersecting polygon in the resultant geometries which happens to be the case precisely in Figure 14. We also see in Figures 1 and 2 that our derived limit theorems do hold true, as the graph begins to flatten out for large choices of $\sigma$, as the values go away from $\sigma_*$. We are also optimistic about the fact that these graphs look like weakly unimodal functions unlike them being highly non-convex for the case of the circle and toroidal helix.

\subsection{Experiments with real data:}
In this subsection, we show some experiments with real biological data from four mice, all of which were infected by malaria and treated for cure in experiments conducted by microbiologists and immunologists \cite{Malaria} at Stanford University. The link to the mice dataset from the biological experiment on mice is here: \url{http://journals.plos.org/plosbiology/article?id=10.1371/journal.pbio.1002436#sec024}. Three of the mice survived with treatment while the other did not. For each mouse we collect the data set of several characteristics for survivor mice that we intuitively believe to repeat with time, i.e. are periodic functions of time. This is because certain physical traits of subjects (patient or mice) show repeating pattern at beginning and end of a disease if the patient survives. This however is not the case normally for non-survivor patients as with them we do not see these physical traits to be periodic with time. For example, the red bloodcell count (RBC) of a mouse had a higher value in the beginning period of malaria, and then decreased as the disease got severe, and eventually increased again with the treatment and came back to normal. As another characteristic, the bacterial count was less in the beginning, but with the disease being severe, it increased and then dropped again with treatment. \par

These periodic behaviors of certain, say $d$ number of physical traits with approximately or exactly equal periods of a survivor patient imply that when plotted not against time, but against each other, they will most probably form a loop structure in $\mathbb{R}^d$. As a model example, one can think of the pair of periodic functions $sin(t), cos(t)$, which form the unit circle when plotted against each other. However, if the number of traits $d$ is bigger than $3$, it is hard to conclude from the data set whether they form a loop or not, and hence whether the patient is a survivor or not. This is where we apply our fast dimensionality reduction technique to obtain two dimensional projections for varying bandwidth $\sigma$, and check whether for the energy minimizing bandwidth mentioned in section $5$, we obtain a polygon. If we do, we will conclude that the patient was survivor, otherwise non-survivor. We can obtain this energy-minimizing bandwidth by either searching through a grid of discrete choices of $\sigma$ or through Algorithm 1 above. As part of some pre-processing we apply Lowess local regression to each variable in order to smooth the data with small smoothing parameter $\alpha$ that leads to utilization of a smaller proportion of total data points while performing local regression.\par
In our experiments with mice data, the number of samples (points) considered are $n=25$ and the number of attributes considered are $d=4$. The four attributes are logarithm of parasite density, red blood cell count, temperature and weight of the mouse under consideration.\par
For survivor mouse $3$, the plot of $log(\sigma)$ vs energy $E(\sigma)$ is given in Figure 12. We notice that at $\sigma_*=3$ the energy $E(\sigma)$ is minimized and we show in Figure 13 that the corresponding two dimensional projection $Z(\sigma_{*})$ is indeed a non self-intersecting polygon. 

\begin{figure*}[ht]
\begin{multicols}{2}
     \includegraphics[height=0.6\linewidth]{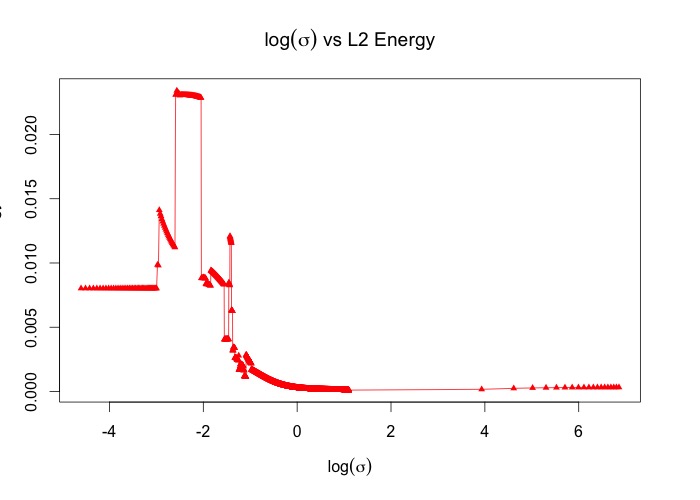} \caption{ $\protect log(\sigma)$ Vs. $\protect \mathbf{E(\sigma)}$ plot for surviving mouse 3 }     
    \includegraphics[height=0.6\linewidth]{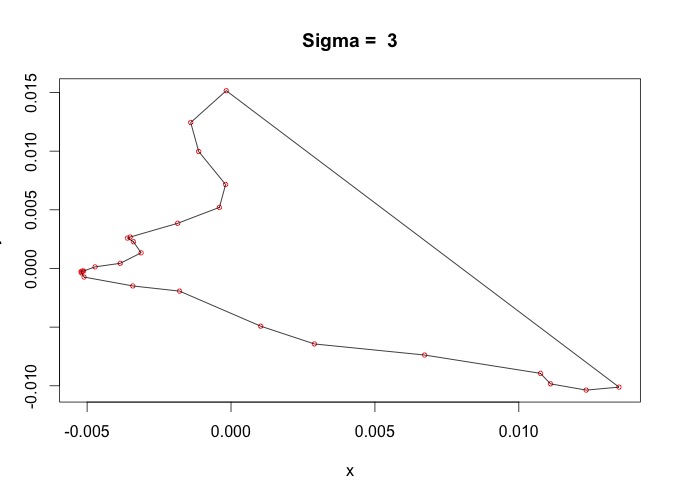} 
    \caption{(Mouse 3) Non self-intersecting polygon at $\protect \sigma_{\ast} = 3$ }
\end{multicols}
\begin{multicols}{2}
    \includegraphics[width=0.7\linewidth]{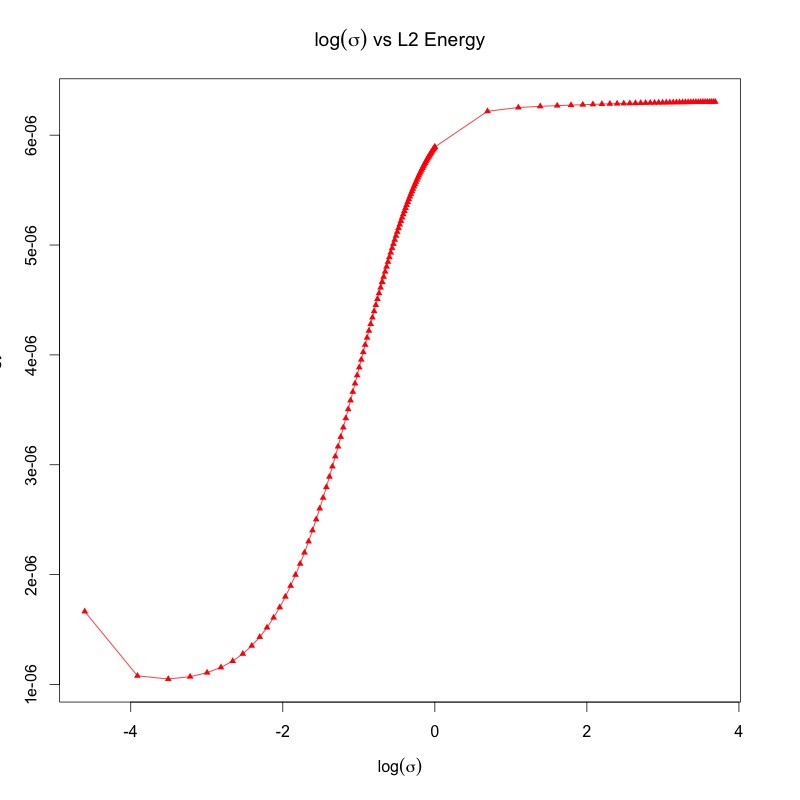} 
    \caption{$\protect log(\sigma)$ Vs. $\protect \mathbf{E(\sigma)}$ plot for surviving mouse 2} 
     \includegraphics[width=0.75\linewidth]{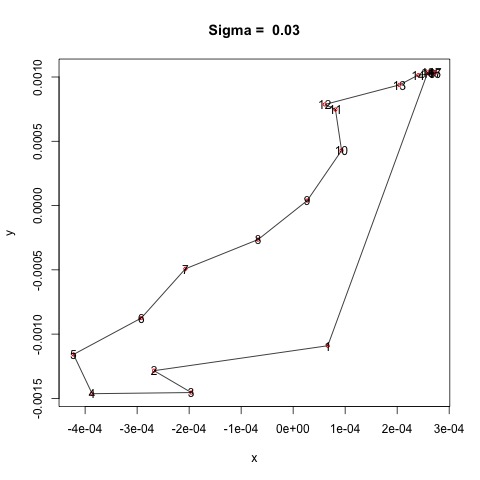} 
    \caption{(Mouse 2) Non self-intersecting polygon at $\protect \sigma_{\ast} = 0.03$ } 
\end{multicols}
\begin{multicols}{2}
 \includegraphics[width=0.7\linewidth]{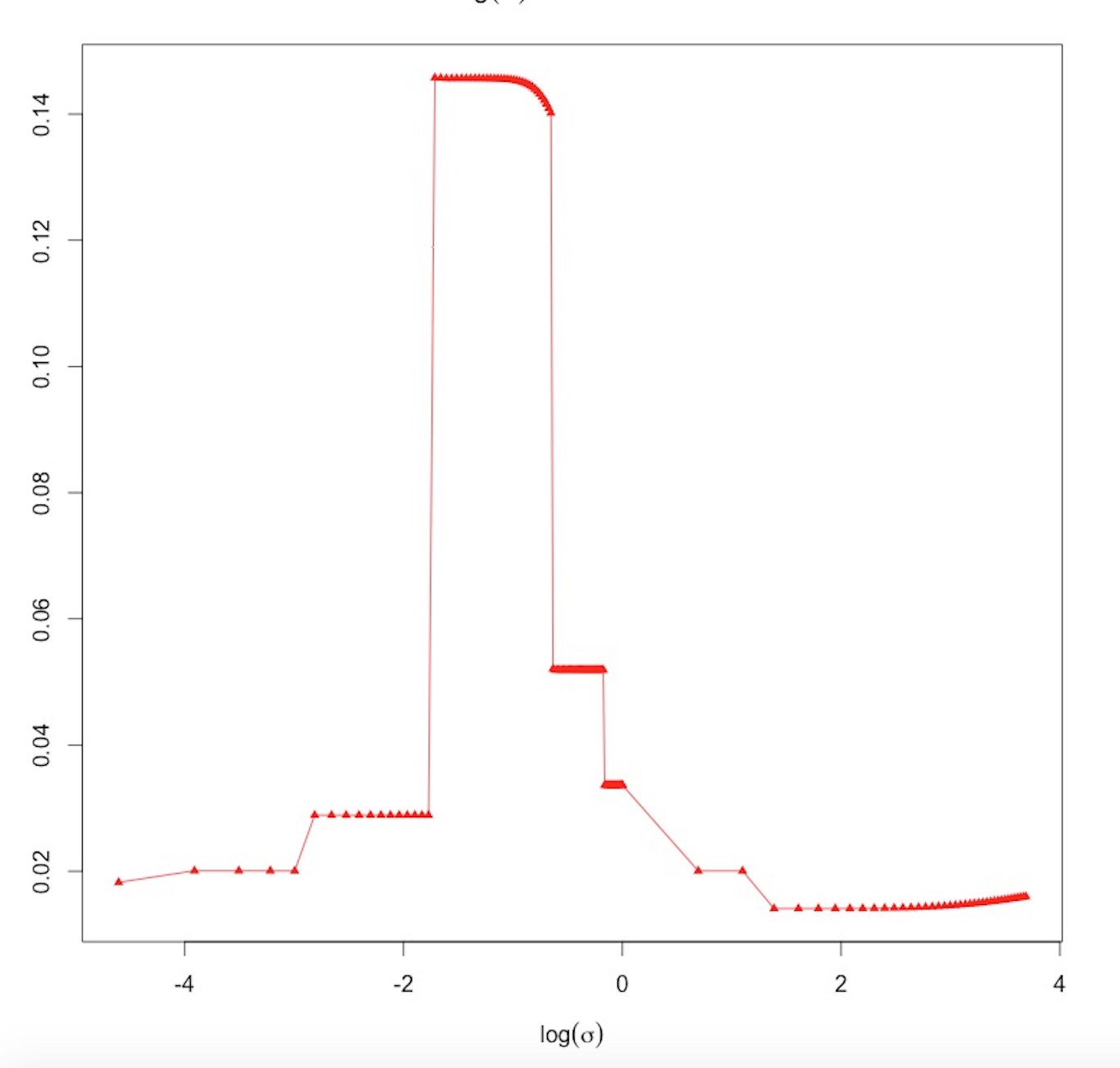} 
 \caption{$\protect log(\sigma)$ Vs. $\protect \mathbf{E(\sigma)}$ plot for non-surviving mouse}
    \includegraphics[width=0.7\linewidth]{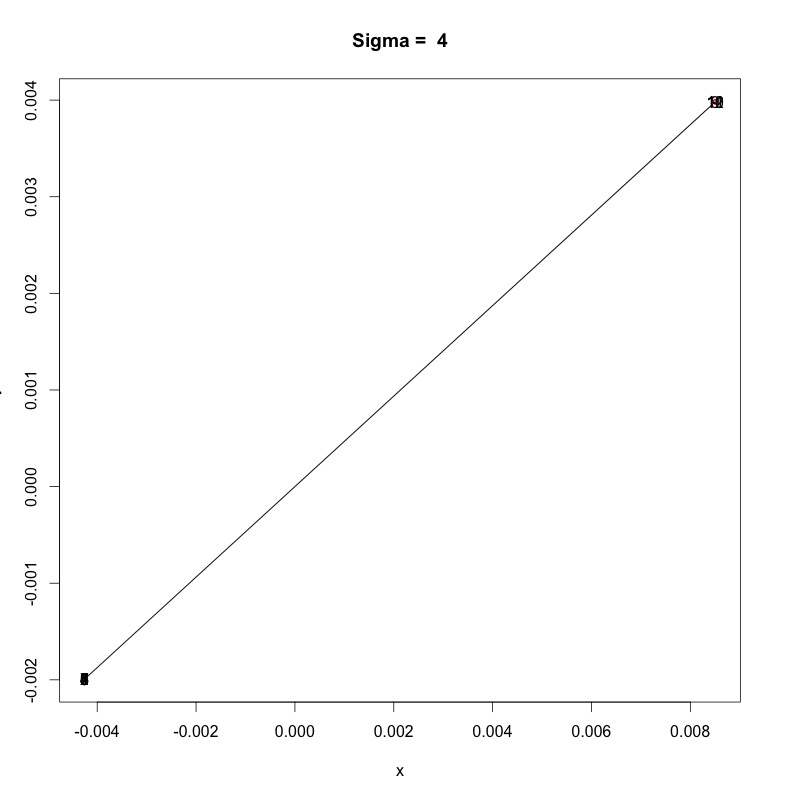} 
    \caption{Non self-intersecting polygon at $\protect \sigma_{\ast} = 4$}
\end{multicols}
\end{figure*}
For survivor mouse $2$, the energy $E(\sigma)$ vs log(bandwidth $\sigma$) is shown in Figure 14. Here, $\sigma_*=0.03$ minimizes the energy $E(\sigma)$ and we see in Figure 15 that the corresponding two dimensional projection $Z\sigma_{*}$ is indeed a non self-intersecting polygon. Here the graph isn't visibly asymptotic for the survivor mouse 2 near the origin, but that's most probably caused by the lack of an ideal (as in exact) periodic pattern in higher dimensions. After the survivor mice experiments we now show and contrast the energy plot for the non-survivor mouse.\\
Quite interestingly, at $E(\sigma)$ minimizing $\sigma = \sigma_* = 4$ when we plot $E(\sigma_*)$ vs $log(\sigma_*)$ in Figure 16, the graph looks very different than the  survivor mice data, or even the energy plots obtained for synthetic data lying on a circle or a toroidal helix as shown in Figures 1 and 2, significant parts of which show convexity. The plot for the non-survivor mice is not convex and furthermore, we notice that 
$\sigma_{*}=4$ minimizes the energy $E(\sigma)$, but the corresponding $Z(\sigma_*)$ is shown in Figure 17. We note that $Z(\sigma=4)$ here is not a non-self intersecting polygon, and hence by our main result we conclude that the test data for the physical traints of the non-survivor mouse was not periodic with time, indicating that the mouse was not a survivor, which matches up with the fact. \par

 Finally, we would also like to state that the computation of our proposed $\mathbf{E}(\mathbf{Z}(\sigma))$ only requires $O(n)$ operations per choicse of $\sigma$, espe
 cially as they are computed on an ordered set of points. Also, each column of $Z(\sigma)$ can be computed in nearly $O(mlog^{1/2}(n))$ time for each choice of $\sigma$ where $m$ is the number of non-zero entries. This complexity is due to the fact that we can obtain the fast manifold learning map $Z(\sigma)$ for any choice of $\sigma$ by solving a symmetric diagonally dominant (SDD) linear system of equations \cite{FastSDD}.
 
\section{Conclusion and Future Work:}

In this paper, we provide a test to check for circularity or periodicity in high dimensional data sets by looking at their nonlinear projections into $\mathbb{R}^2$ given by the fast manifold learning map introduced in \cite{FastSDD}. There are several directions in which such tests can be generalized to. For example, if the pattern we expect in the high-dimensional data is topologically more complex but still has one dimension, for example, say as a wedge of circles, it would be interesting to see if our projection method can detect the corresponding topological structures after the nonlinear projection. Another direction could be detecting topological structures with two dimensions, for example of that of a sphere or a torus, by a modification of our projection method onto two dimensions. We would also like to investigate speeding up the computations for finding the minimizer of $\mathbf{E(Z(\sigma))}$ as the manifold learning map $\mathbf{Z}(\sigma)$ already involves fast computations as in \cite{FastSDD}.

\end{document}